\documentclass{article}

\usepackage{microtype}
\usepackage{xcolor}
\usepackage{graphicx}
\usepackage{subfigure}
\usepackage{booktabs} % fsee professional tables

\usepackage{hyperref}

\usepackage[utf8]{inputenc} % allow utf-8 input
\usepackage[T1]{fontenc}    % use 8-bit T1 fonts
\usepackage{hyperref}       % hyperlinks
\usepackage{url}            % simple URL typesetting
\usepackage{booktabs}       % professional-quality tables
\usepackage{amsfonts}       % blackboard math 
\usepackage{amsmath}
\usepackage{amssymb}
\usepackage{nicefrac}       % compact symbols for 1/2, etc.
\usepackage{microtype}      % microtypography

\usepackage{algorithm}
\usepackage{algorithmic}
\usepackage{natbib}
\usepackage{amsthm}
\usepackage{bbm}
\renewcommand{\algorithmicrequire}{\textbf{Input:}}
\renewcommand{\algorithmicensure}{\textbf{Output:}}
\newtheorem{theorem}{Theorem}[section]
\newtheorem{corollary}{Corollary}[theorem]
\theoremstyle{definition}

\def\P{\mathbb P}
\def\E{\mathbb E}

\def\argmax{\mathrm{argmax}}
\def\argmin{\mathrm{argmin}}

\definecolor{mygreen}{rgb}{0.188,0.658, 0.178}

\usepackage[accepted]{icml2019}

\icmltitlerunning{Best Arm Identification in Linked Bandits}
\begin{document}

\twocolumn[
\icmltitle{Best Arm Identification in Linked Bandits}



\icmlsetsymbol{equal}{*}

% \begin{icmlauthorlist}
% \icmlauthor{Anant Gupta}{to}
% % \icmladdress{Department of Computer Sciences, University of Wisconsin-Madison}
% \end{icmlauthorlist}
 \begin{center}{\bf
 Anant Gupta} \\
Department of Computer Sciences\\
  University of Wisconsin-Madison\\
  \texttt{anant@cs.wisc.edu}
\end{center}

% \icmlaffiliation{to}{Department of Computer Sciences, University of Wisconsin-Madison}
% \icmlaffiliation{goo}{Googol ShallowMind, New London, Michigan, USA}
% \icmlaffiliation{ed}{School of Computation, University of Edenborrow, Edenborrow, United Kingdom}

% \icmlcorrespondingauthor{Anant Gupta}{anant@cs.wisc.edu}
% \icmlcorrespondingauthor{Eee Pppp}{ep@eden.co.uk}


\vskip 0.3in
]




\begin{abstract}
  We consider the problem of best arm identification in a variant of multi-armed bandits called linked bandits. In a single interaction with linked bandits, multiple arms are played sequentially until one of them receives a positive reward. Since each interaction provides feedback about more than one arm, the sample complexity can be much lower than in the regular bandit setting. We propose an algorithm for linked bandits, that combines a novel subroutine to perform uniform sampling with a known optimal algorithm for regular bandits. We prove almost matching upper and lower bounds on the sample complexity of best arm identification in linked bandits. These bounds have an interesting structure, with an explicit dependence on the mean rewards of the arms, not just the gaps. We also corroborate our theoretical results with experiments.
\end{abstract}

% The \author macro works with any number of authors. There are two
% commands used to separate the names and addresses of multiple
% authors: \And and \AND.
%
% Using \And between authors leaves it to LaTeX to determine where to
% break the lines. Using \AND forces a line break at that point. So,
% if LaTeX puts 3 of 4 authors names on the first line, and the last
% on the second line, try using \AND instead of \And before the third
% author name.

\newcommand{\norm}[1]{\left\lVert #1 \right\rVert}

\section{Introduction}
Best arm identification in multi-arm bandits is a well-studied problem, for which provably optimal algorithms exist. Although many real world problems can be cast in this framework, they don't always provide the flexibility of independent sampling across different arms. This paper considers a variant to the best arm identification problem in the stochastic multi-arm bandit (MAB) setting, called \emph{ordered linked} bandits. Consider a MAB setting with $n$ arms, that have binary rewards with unknown means $\mu_1, ..., \mu_n$ in $[0,1]$. The arms are assumed to be ordered in an arbitrary but fixed, known way.
% A sample of the $i$th arm provides an independent realization of its reward, which is a Bernoulli random variable with mean $\mu_i$.
In linked bandits, the action space of the agent is different from the standard formulation. Instead of selecting a single arm to sample at one time, the agent selects an ordered \emph{subsequence} of candidate arms to be sampled together. The selected arms are then sampled sequentially until one of them obtains a reward of $1$, and the remaining arms (if any) are \emph{unobserved}. Therefore, the sampling of arms happens in a stochastic manner, dependent on rewards obtained for previously sampled arms. The goal of the agent is to identify the arm with the highest mean, using as few rounds as possible. A single round of our learning problem provides richer feedback than the standard version, since many arms are sampled at once. At the same time, feedback is uncertain and entangled across different arms.

This setting is primarily motivated by the problem of active 3D sensing in computer vision \cite{shin2015photon}. In active 3D sensing, the depth of a scene point is estimated by projecting a pulse of light, and estimating the maxima in the temporal intensity function reflected off the point. The sensor makes measurements of the unknown intensity $\mu_i$ at discrete time intervals $a_i$, where each time interval can be thought of as an arm of a bandit. Moreover, the measurement at time $a_i$ is a Bernoulli random variable with mean equal to the intensity $\mu_i$ at that time. A single round of measurement ends when a positive signal is detected, after which the sensor resets and the measurement process repeats. The goal is to select a subset of times for which the sensor should make a measurement, and identify the time $a^*$ of highest intensity using as few measurements as possible. 
% corresponds to Photon detection is a stochastic process, and the sensor makes discrete measurements across time of photon arrivals. Corresponding to each time bin, the sensor records a Bernoulli random variable with mean $\mu_i$ proportional to the intensity at time bin $i$. Due to certain physical constraints, the sensor is constrained to detect at most $1$ photon in each round of measurement. The goal is to identify the time bin with the highest mean (since it corresponds to the time taken for light reflected from the scene to reach the sensor), and thus recover the depth.
% By actively gating the sensor at each time bin, we cast the problem as an instance of linked bandits. \\

% Another potential application is in search ranking (\cite{radlinski2008learning}). Given $n$ documents, we want to identify the top-$k$ results to display to each user. When the user clicks on a result, every document preceding that gets a reward of $0$, while the clicked document gets a reward of $1$. \\
In this work, we study the relation between linked bandits and standard bandits, specifically looking at the problem of best-arm identification in the fixed confidence setting. We show that the structured feedback obtained in linked bandits can be disentangled while preserving independence across arms, allowing us to prove similar concentration bounds as for standard bandits. We pose the problem of designing sampling strategies that provide the maximum feedback, while sampling arms in the same gap dependent proportion as optimal algorithms for standard bandits. Starting with simple sampling strategies that provide insights about the complexity of linked bandits, we go on to design a sample-efficient algorithm that makes the best use of the rich feedback structure in linked bandits.
% , as well as the gap dependent sampling of optimal standard bandit algorithms.
% We derive an upper bound on the sample complexity of our algorithm, and compare it with that of the optimal standard bandit algorithm.
% achieves optimal performance for best-arm identification with linked bandits, and derive matching lower and upper bounds for the problem.
We then prove matching upper and lower bounds on the sample complexity of our problem, which match upto constant and doubly-logarithmic multiplicative factors. We briefly discuss a natural generalization of our problem to the setting where arms can be re-ordered, and show that somewhat surprisingly, the results from our analysis still hold. Finally, we conclude with some experiments that support our theoretical results.
% Using a novel uniform sampling scheme, we extend the exponential gap elimination algorithm to our linked bandits setting and obtain an optimal method that can be used to implement a practical solution for the 3D sensing problem.
% (See if Combinatorial Pure Exploration of
% Multi-Armed Bandits and PAC subset selection in stochastic multi-armed
% bandits are relevant, i.e., can the LUCB idea be extended to linked bandits too?)

\section{Linked Bandits}
\subsection{Problem Setting}
We consider a general sequence of $n$ arms ($a_1, a_2, ..., a_n)$, with unknown mean vector $(\mu_1, \mu_2, ..., \mu_n) \in [0, 1]^n$. We denote the best arm by $a_{i^*}$, where $i^* = \argmax_{i \in [n]} \{\mu_i\}$. With some abuse of notation, we will also use $a^*$ and $\mu^*$ to denote the best arm $a_{i^*}$ and the highest mean $\mu_{i^*}$ respectively. At time $T$, the learning agent select a subsequence of the arms $(a_{i_1^T}, a_{i_2^T}, ..., a_{i_k^T})$, where $1 \leq i_1^T < i_2^T < ... < i_k^T \leq n$. For each of the selected arms, the environment generates independent Bernoulli rewards $x_{i_1^T}, x_{i_2^T}, ..., x_{i_k^T}$. However, only a prefix of the reward subsequence, upto the first $1$ reward is revealed to the agent. If no arm received a $1$ reward, the whole subsequence is revealed to the agent. For ease of notation, we will henceforth drop the subscript $T$ from the arm index. The feedback that the agent receives at time $T$ is $(x_{i_1}, x_{i_2}, ...,x_{i_p})$, where:
\[ p = \begin{cases}
k &\text{if }x_{i_j} = 0 \\
&\quad \text{ for } 1 \leq j \leq k \\
\argmin \{1 \leq j \leq k : x_{i_j} = 1\} &\text{otherwise.}\\
\end{cases}\]
Therefore, at time $T$, arms $(a_{i_1}, ..., a_{i_p})$ get \emph{sampled}, and obtain rewards. Henceforth, we'll refer to the event of selecting a subsequence of arms and receiving feedback at any time $T$ as a ``play''. \\
% \theoremstyle{definition}
% \begin{definition}{Play}
% A play is an ordered sequence of arms $i_1, i_2, ..., i_k$ selected by the player to be sampled, of which a prefix is actually sampled.
% \end{definition}
% Note that an arm's reward might not be revealed every time it is played, since another arm before it in the sequence might have already received a $1$ reward. $t_i$ only counts the plays where the $i$th arm explicitly received either a $0$ or a $1$ reward, or in other words, was actually sampled.
At the end of the game, the agent returns her prediction of the best arm $a_{\hat i}$. The agent's policy can be evaluated in two ways. In the \emph{fixed confidence} setting, a confidence level $\delta$ is specified by the problem, and the agent's policy is evaluated by the number of plays $T_\delta$ issued while ensuring that the probability of error $\mathbb{P} \{ \hat i \neq i^* \} < \delta$. 
% where $\hat i_T$ is the predicted optimal arm after $T$ plays.
In the \emph{fixed budget} setting, a budget of $T$ plays is specified, and the policy is evaluated using the probability of error. Analogous notions are used in the standard MAB formulation, with number of plays replaced by number of samples or arm-pulls. In this paper, we only consider the fixed confidence setting. We refer to the number plays needed by an algorithm to identify the best arm with a given confidence as the \emph{play complexity} of the algorithm.

\subsection{Interpreting linked bandit feedback}
At time $T$, we denote the cumulative reward that was obtained for the $i$th arm by $X_i$, and the number of times the $i$th arm was sampled by $t_i$. The vectors $(X_i)_{i=1}^n$ and $(t_i)_{i=1}^n$ can be computed iteratively as
\begin{align}
\label{update}
    X_i &= \sum_{j=1}^T \left(x_i^j . \mathbbm{1}_{\left\{i_1^j, i_2^j, ..., i_{p^j}^j \right\}}\left(i \right)\right) \nonumber \\
    t_i &= \sum_{j=1}^T \mathbbm{1}_{\left\{i_1^j, i_2^j, ..., i_{p^j}^j \right\}}\left(i \right) \,,
\end{align} 
where $\mathbbm{1}_A(x)$ denotes the indicator function for the set $A$. Note that $X_i | t_i \sim \text{Bin}(t_i, \mu)$, where $\text{Bin}(N, p)$ is a binomial random variable with $N$ trials and mean $p$.
Using this fact, we can compute an empirical mean estimate for the $i$th arm as
\[ \hat \mu_i = \frac{X_i^t}{t_i^t}\,. \]
Note that
\begin{align*}
\E{(\hat \mu_i)} &= \E{(\E{(\hat \mu_i | t_i)})} \\
&= \E{(\frac{\mu_i t_i}{t_i})} = \mu_i \,.
\end{align*}
Therefore, one can hope to bound $\hat \mu_i$ around $\mu_i$ with high probability via a concentration inequality. Even though the samples $t_i^t$ are not independent, as long as $t_i^t$ is large enough for all $i$, the empirical means concentrate around the true means, and the arm $a^*$ with the highest mean $\mu^*$ can be identified.
% This is possible because $\mu_i$ is conditionally independent of other $\mu_j$'s, given $t_i$.
The goal of the agent then is to design a strategy for selecting subsequences of arms to sample, such that they are all sampled enough times with high probability.

% The agent observes the sequence of rewards $(0,0,...,1)$ for the arms that are actually sampled upto and including the first instance of a $1$ reward. Therefore, the player receives an unbiased estimate of the means of a random prefix of the selected subsequence.
To minimize the number of plays needed, the agent has to select a subsequence that is expected to provide the most useful feedback. However, the feedback received by the agent in a play is a function of the stochasticity of the environment. Therefore, the agent does not know a priori which arms or how many will get sampled in a play. All she can do is select a subsequence of arms that need to be sampled the most, and hope that a large number get sampled.
% Note that there is an inherent trade-off in the size of the subsequence that should be selected. On the one hand, selecting a longer subsequence can provide feedback for more arms. On the other hand, always playing the longest possible subsequence will mean that latter arms are sampled infrequently (since the probability of a long sequence of $0$s arising decreases exponentially with length).

\section{Sampling strategies}
\subsection{Comparison with standard bandits}
It is useful to get an idea about the sample complexities that our designed strategies in the linked bandit setup can be expected to achieve vis-a-vis the standard MAB formulation.Since the linked bandit setup can simulate the standard setup by simply choosing a single arm in each play, therefore the complexity is trivially upper bounded by that of standard MAB. However, we can try to do better, since each play provides more information than just one sample. The improvement over standard MAB would depend on how many arms are sampled per play on average, which further depends on the means $\mu_i$. For low values of $\mu_i$'s, the sequence of arms sampled will be longer on average, and we can expect a large improvement. If all $\mu_i$'s are close to $1$, then each play will only sample around $1$ arm on average, and we can't expect much of an improvement. In summary, the sample complexity for linked bandits can be expected to show a mean-dependent improvement over standard MAB. \\

\subsection{Maximal sampling algorithm}
To maximize the number of arms that can be sampled per play, we can consider a strategy that always selects the whole sequence of arms at each time. This would be a good strategy when the mean probabilities are small, since each play can potentially sample all $n$ arms at once. In that case, this strategy would be similar to the uniform allocation strategy for standard MAB \cite{BUBECK20111832}, which samples all arms uniformly in a round-robin fashion.
\begin{algorithm}
\caption{\texttt{MaximalSampling}}
\begin{algorithmic} 
\REQUIRE{sequence of arms $S = \{a_1, a_2, ..., a_n \}$ \\
number of plays $T$}
\STATE \textbf{Initialize:} reward and sample counts $(X_i)_{i=1}^n = 0, (t_i)_{i=1}^N = 0$
\STATE Play $\{a_1, a_2, ..., a_n\}$ $T$ times
\STATE Update $(X_i)_{i=1}^n,(t_i)_{i=1}^n$
using Eq.~\ref{update}
\STATE $(\mu_i)_{i=1}^n \leftarrow \left(\frac{X_1}{t_1}, \frac{X_2}{t_2}, ..., \frac{X_n}{t_n}\right)$
\ENSURE arm $a_{\hat i}$ where $\hat i = \argmax_i{\mu_i}$
\end{algorithmic}
\end{algorithm}
% Here, we instead play the whole sequence of arms together, and hope that each arm will get sampled enough number of times. 
% Clearly, we shouldn't expect this strategy to perform very well, as the latter arms will get sampled with exponentially small probabilities. 
The following result provides an upper bound on the sample complexity of this strategy: \\
\begin{theorem}
With probability at least $1 - \delta$, the \texttt{MaximalSampling} strategy finds the optimal arm using at most
\[ O\left(\frac{ 1}{\Delta^{2}p^2} \right) \log{\frac{n}{\Delta p\delta}} \]
plays, where $p = \prod_{i=1}^{n-1} (1 - \mu_i)$.
\end{theorem}
The proof is discussed in the Appendix. The proof has three main steps. First, we show that the maximum empirical mean of sub-optimal arms concentrates around its expectation, using a novel method of bounded differences with high probability. Second, we bound the expected maximum empirical mean by an $O(\sqrt{\ln{n}})$ term, just as for the expected maximum of $n$ independent gaussians. Third, we bound the probability that a sub-optimal arm has the highest empirical mean using a large enough $T$. Notably, the first two steps work despite the fact that the empirical means in linked bandits are dependent, unlike standard bandits.\\
% The main challenge is showing that the empirical mean estimates $\hat \mu_i$ concentrate around the true means $\mu_i$, even though the samples across arms are not independent. This is done using the method of bounded differences.\\
% Even though this is an upper bound on the sample complexity, it is fairly indicative of the actual complexity in practice.
% For standard MAB, exponential gap elimination \cite{karnin2013almost} is an optimal algorithm that uses uniform allocation sampling as a subroutine, and achieves a sample complexity of $O\left(\sum_{i\neq i^*} \Delta_i^{-2} \log{{\delta^{-1}} \log{{\Delta_i^{-1}}}}\right)$. Therefore, we can adapt this algorithm to our linked bandits variant, with uniform allocation replaced by the ``whole sequence'' strategy. An upper bound similar to Theorem 1 can be proved for this algorithm that combines exponential gap elimination with the minimum wastage principle
% (more explicitly, for each round of exponential gap elimination, we play the whole sequence of remaining arms together until each arm has been effectively sampled for the number of rounds required by the algorithm).
Note that the sample complexity of this strategy differs from that of the uniform allocation strategy for standard MAB \cite{BUBECK20111832} by a divisive factor of $f = n \prod_{i=1}^{n-1} (1 - \mu_i)^2$. For small values of $\mu_i$, $f \approx n$ and the sample complexities are good. For large values of $\mu_i$ however, $f$ can be much smaller than $1$ (exponential in the number of arms $n$), and the sample complexities are actually much worse than those of standard MAB. The reason for the high complexity is that the whole sequence allocation strategy doesn't sample all arms uniformly, and is biased against the latter arms. The probability that the last arm in the sequence gets sampled is only $\prod_{i=1}^{n-1} (1 - \mu_i)$, and a large number of plays is needed to sample it sufficient number of times. Clearly, we need a better allocation strategy that doesn't have a bias depending on the position of the arm in the sequence.

\subsection{Subroutine for uniform sampling}\label{suffixsample}
The previous section motivated the design of a sampling strategy which samples all arms fairly and uniformly, while still being maximal. In this section, we introduce a strategy to perform uniform sampling which meets these requirements. 
% As we will see later, using this strategy as a subroutine leads to the design of algorithms that are optimal in terms of sample complexity.
\begin{algorithm}
% \caption{SuffixSample ($S, t$)}
\textbf{function} {SuffixSample}($S, t$) 
\begin{algorithmic} 
% \REQUIRE{sequence of arms $S = \{a_1, a_2, ..., a_n \}$ \\
% number of samples $t$}
    \STATE \textbf{Initialize:} sample counts $t_1 = t_2 = ... = t_n = 0$
    \STATE $i \leftarrow 1$
    \WHILE{$i \leq n$}
    \STATE Play $\{a_i, a_{i+1}, ..., a_n\}$ $t - t_i$ times
    \STATE Update $t_1, t_2, ..., t_n$
    \ENDWHILE \\
    \STATE \textbf{return} Empirical estimates $\mu_1, \mu_2,..., \mu_n$
\end{algorithmic}
\end{algorithm}
The sampling algorithm proceeds iteratively by playing a suffix of the sequence of arms till the leftmost arm is sampled $t$ times in total, then removing the leftmost arm from the suffix. This ensures that each arm is sampled exactly $t$ times when the algorithm ends. Therefore, the algorithm meets the first requirement. Moreover, since each play samples as many arms as possible, the algorithm is maximal and has a low play complexity.

% \subsubsection{Play complexity of SuffixSample}
The total number of plays that the algorithm uses is stochastic, but can be bounded with high confidence around its mean. Let $X_1, X_2, ..., X_n$ denote the number of times a reward of $1$ was obtained for $a_1, a_2, ..., a_n$. The algorithm maintains the invariant that after $i$ rounds, arms $a_1, a_2, ..., a_i$ have all been sampled $t$ times. The first round of the algorithm has $t$ plays. After $i$ rounds, $a_{i+1}$ has already sampled $t - X_i$ number of times, therefore to maintain the invariant, the $(i+1)$th round has $X_i$ plays. Therefore, the total number of plays is given by:
\begin{equation}
    \label{tsum}
    T = t + X_1 + X_2 + ... + X_{n-1}
\end{equation}
Each $X_i$ is binomially distributed with mean $\mu_i$ and number of trials $t$, conditioned on $X_1, X_2, ..., X_{i-1}$. Since $\P(X_i | X_1, X_2, ..., X_{i-1})$ doesn't depend on $X_1, X_2, ..., X_{i-1}$, we have $X_i \perp X_j$ for all $i,j$. Using Hoeffding's inequality, we get:
\begin{equation} \label{eq1} \mathbb{P}\left\{T - t(1 + \mu_1 + ... + \mu_{n-1}) > \epsilon\right\} \leq \exp{\left(-\frac{ 2\epsilon^2}{t (n - 1)}\right)}\end{equation}
where $T$ is the total number of plays. Therefore, the algorithm  samples each arm $t$ times, using close to $t\left(1 + \sum_{i=1}^{n-1} \mu_i\right)$ plays in total with high probability. Moreover, the samples are independent, which makes an invocation of \texttt{SuffixSample} in linked bandits equivalent to round robin uniform sampling in regular bandits.

\subsection{Uniform sampling algorithm: using \texttt{SuffixSample}}
The subroutine introduced in the previous section naturally leads to an algorithm for best arm prediction:
\begin{algorithm}
\caption{\texttt{UniformSampling}}
\begin{algorithmic} 
\REQUIRE{
sequence of arms $S = \{a_1, a_2, ..., a_n \}$, number of samples $t$}
% \STATE $t \leftarrow \max \left\{n, \frac{n \ln{n}}{\eta^2 \Delta^2}\right\}$
\STATE $\mu_1, \mu_2,..., \mu_n \leftarrow $ SuffixSample$(S, t)$
\STATE $\hat i \leftarrow \argmax_i{\mu_i}$
\ENSURE arm $a_{\hat i}$
\end{algorithmic}
\end{algorithm}

\begin{theorem}
With probability at least $1 - \delta$, uniform allocation for linked bandits finds the optimal arm using at most
\[ O\left(\left(\frac{1 + \sum_{i=1}^{n-1} \mu_i}{\Delta^2} + \sqrt{\frac{n-1}{\Delta^2}}\right) \log{\frac{n}{\delta}}\right) \]
plays.
\end{theorem}
\begin{proof}
For all $\eta \in (0,1)$, if arms $(a_i)_{i=1}^n$ are each sampled independently $t \geq 1 + \frac{2\ln{n}}{\eta^2 \Delta^2} $ times, we have:
\begin{equation} \label{eq2} \mathbb{P}\{ \hat i \neq i^*\} \leq \exp{\left( - \frac{(1-\eta)^2}{2} t \Delta^2 \right)} \end{equation}
where $\hat i = \text{arg} \max_{i} \hat \mu_i$.
This follows from the second statement of Proposition 1 in \cite{BUBECK20111832}. We'll use $\delta/2$ confidence for bounding the probability of error in Eq.~\ref{eq2}. To make $\mathbb{P}\{ \hat i \neq i^*\} \leq \delta/2$, we need $t \geq \frac{2}{(1 - \eta)^2} \frac{1}{\Delta^2} \log{\frac{2}{\delta}}$. Since we also need $t \geq 1 + \frac{2\ln{n}}{\eta^2 \Delta^2} $, it suffices to sample each arm:
\begin{equation} \label{eq2.5} t = O\left(\frac{1}{\Delta^2}\log{\frac{n}{\delta}}\right)
\end{equation}
times.\\
From Eq.~\ref{eq1}, we can get an upper bound on the number of plays $T$ needed to sample each arm $t$ times with high probability.  We'll use the remaining $\delta/2$ confidence here.
Setting $\epsilon = \frac{(1 - \eta) t \Delta \sqrt{n - 1}}{2}$, we get:
\begin{align*} 
\mathbb{P}\left\{T  > t\left(1 + \sum_{i=1}^{n-1}\mu_i + (1 - \eta) \Delta \sqrt{n-1} / 2\right) \right\} \leq \\
\exp{\left( - \frac{(1-\eta)^2}{2} t \Delta^2 \right)} \leq \delta/2
\end{align*}
Combining this with Eq.~\ref{eq2.5}, we get $T = O\left(\left(\frac{1 + \sum_{i=1}^{n-1} \mu_i}{\Delta^2} + \sqrt{\frac{n-1}{\Delta^2}}\right) \log{\frac{n}{\delta}}\right)$ with probability $\geq 1 - \delta/2$.
% Combining this with \ref{eq2}, we get:
% \begin{equation} \label{eq3} \mathbb{P}\{ \hat i \neq i^*\} \leq 2\exp{\left( - \frac{(1-\eta)^2}{2} \frac{T}{1 + \sum_{i=1}^{n-1} \mu_i + (1-\eta) \Delta \sqrt{n-1} /2} \Delta^2 \right)} \end{equation}
\end{proof}
{\bf Comparison with {\tt MaximalSampling}: } Ignoring the $O(1/\Delta)$ term and logarithmic factors, the sample complexity of \texttt{UniformSampling} is order $(1 + \sum_{i=1}^{n-1} \mu_i)\Delta^{-2}$. It is easy to see that this is always less than ${\prod_{i=1}^{n-1} (1 - \mu_i)^{-2}}\Delta^{-2}$, the complexity of Maximal Sampling, even for small $\mu_i$. Thus, even though {\tt MaximalSampling} can theoretically sample more arms per play than {\tt UniformSampling}, the extra samples don't help. This is due to the skewed distribution of samples across arms. \\
{\bf Comparison with standard bandits: } {\tt UniformSampling} uses a sampling strategy that effectively samples each arm uniformly using the least possible number of plays. Compared to the uniform strategy for regular bandits which has a sample complexity of $\frac{n}{\Delta^2}$, this is as big an improvement as we could have hoped for. The relative improvement over standard MAB increases from $1$ to $n$ as $\sum_{i=1}^{n-1}\mu_i$s decreases from $n-1$ to $0$, in agreement with our intuitive expectation. However, the optimal sampling strategy for regular bandits is not uniform. Arms with large gaps don't need to be sampled as many times as arms with small gaps, and can be eliminated much earlier. We need to design a an algorithm for linked bandits that exploits this fact.

\section{Exponential gap elimination: using \texttt{SuffixSample}}
For standard MAB, exponential gap elimination \cite{karnin2013almost} is a known optimal algorithm that uses uniform sampling as a subroutine, and achieves a sample complexity of $O\left(\sum_{i\neq i^*} \Delta_i^{-2} \log{{\delta^{-1}} \log{{\Delta_i^{-1}}}}\right)$.
In this section, we combine the uniform sampling subroutine that we developed for linked bandits with the exponential gap elimination algorithm for standard bandits. For clarity, we provide the original algorithm (in the context of linked bandits) and the modified one side-by-side.

The major change is in the invocation of sampling routines. Note that there are 2 places where arms are sampled in the algorithm, lines 5 and 6. In both places in the original algorithm, subsets of arms are sampled uniformly 
(\texttt{MedianElimination} also performs uniform sampling in each round internally). For linked bandits, we use our SuffixSample subroutine in both these places, keeping the effective number of times that the arms are sampled the same. As shown in Section ~\ref{suffixsample}, an invocation of \texttt{SuffixSample(S, t)} is equivalent to sampling each $a \in S$ individually, $t$ times. This means that the correctness guarantees of the original algorithm continue to hold for our linked bandits version.
% \begin{minipage}[t]{6.9cm}
% \vspace{0pt}
\begin{algorithm}[H]
\caption{\texttt{EGE} (regular bandits)}
\begin{algorithmic}[1] 
\STATE \textbf{input} confidence $\delta > 0$
\STATE initialize $S_1 \leftarrow [n], r \leftarrow 1$
\WHILE{$|S_r| > 1$}
\STATE let $\epsilon_r = 2^{-r}/4$ and $\delta_r = \delta/(50 r^3)$
\STATE {\color{red} sample each arm $i \in S_r$ for \\$t_r = (2/\epsilon_r^2) \ln{(2/\delta_r)}$ times},\\ and let $\hat \mu_i^r$ be the average reward
\STATE invoke $i_r \leftarrow $ {\color{red}MedianElimination}$(S_r, \epsilon_r/2, \delta_r)$ and let $\hat \mu_*^r = \hat \mu_{i_r}^r$
\STATE set $S_{r+1} \leftarrow S_r \setminus \{ i \in S_r : \hat \mu_i^r < \hat \mu_*^r - \epsilon_r \}$
\STATE update $r \leftarrow r + 1$
\ENDWHILE
\STATE \textbf{output} {\color{red}arm in $S_r$}
\end{algorithmic}
\end{algorithm}
% \end{minipage}
% \begin{minipage}[t]{6.9cm}
% \vspace{0pt}
\begin{algorithm}[H]
\caption{\texttt{LinkedEGE} (linked bandits)}
\begin{algorithmic}[1] 
\STATE \textbf{input} confidence $\delta > 0, {\color{mygreen}\Delta > 0}$
\STATE initialize $S_1 \leftarrow [n], r \leftarrow 1$
\WHILE{$|S_r| > 1$ {\color{mygreen}\AND $r \leq \left\lceil\log_2{\frac{1}{\Delta}} \right\rceil$}}
\STATE let $\epsilon_r = 2^{-r}/4$ and $\delta_r = \delta/(50 r^3)$
\STATE {\color{mygreen}SuffixSample($S_r, (2/\epsilon_r^2) \ln{(2/\delta_r)}$)}, \\and let $\hat \mu_i^r$ be the average reward
\STATE invoke $i_r \leftarrow $ {\color{mygreen}MedianEliminationWith-\\SuffixSample}$(S_r, \epsilon_r/2, \delta_r)$ \\
and let $\hat \mu_*^r = \hat \mu_{i_r}^r$
\STATE set $S_{r+1} \leftarrow S_r \setminus \{ i \in S_r : \hat \mu_i^r < \hat \mu_*^r - \epsilon_r \}$
\STATE update $r \leftarrow r + 1$
\ENDWHILE
\STATE \textbf{output} {\color{mygreen} arm $i_r$}
\end{algorithmic}
\end{algorithm}
% \end{minipage}

The other change in the algorithm is that the number of rounds is limited to at most $\left\lceil\log_2{\frac{1}{\Delta}} \right\rceil$. This does not affect the correctness of the algorithm, and is needed to get a tight bound on the play complexity as discussed later. The cost of this modification is that we now need to know $\Delta$ before hand. We'll now analyze the play complexity of the modified algorithm.

\subsection{Play Complexity Analysis}
\label{ege}
\begin{theorem}
\label{thege}
With probability at least $1 - \delta$,\texttt{LinkedEGE} finds the best arm using at most
\[
\begin{split} &O\left( \frac{1}{\Delta^2} \log{\left(\frac{1}{\delta} \log{\frac{1}{\Delta}}\right)} + \sum_{i \neq i^*} \frac{\mu_i}{\Delta_i^2} \log{\left(\frac{1}{\delta} \log{\frac{1}{\Delta_i}}\right)} + \right. \\
&\quad \left. \sqrt{\sum_{i \neq i^*} \frac{1}{\Delta_i^2} \log{\left(\frac{1}{\delta} \log{\frac{1}{\Delta_i}}\right)}} \log{\frac{1}{\delta}} \right)
\end{split}
\]
plays.
\end{theorem}
\begin{proof}
We borrow elements of the proof from the original paper that proposed exponential gap elimination. The first part is to show that the algorithm returns the optimal arm. The second part is to show an upper bound on the number of plays used. Since the number of effective samples for each arm remains the same in the modified algorithm, Lemma 3.3 and 3.5 still hold. We state them here for reference. \\ \\
\textbf{Lemma 3.3.} With probability at least $1 - \delta / 5$, the optimal arm $a_* \in S_r$ for all $r$. \\\\
% \begin{proof}
% See Lemma 3.3. in \cite{karnin2013almost}.
% \end{proof}
For all $0 \leq s \leq s_{\text{max}} \equiv \lceil \log_2{(1/\Delta)} \rceil$, let $A_s = \{ i \in [n] : 2^{-s} \leq \Delta_i < 2^{-s + 1} \}$, $n_s = |A_s|$ and $S_{r,s} = S_r \cap A_s$, $n_{r,s} = |S_{r,s}|$. \\ \\
% \textbf{Lemma 3.4.} Assume that the optimal arm is not eliminated by the algorithm. Then with probability at least $1 - 4 \delta/5$, we have $|S_{r,s}| \leq \frac{1}{8} | S_{r-1,s} |$ for all $1 \leq s \leq r$. \\ \\
\textbf{Lemma 3.5.} With probability at least $1 - \delta$, the total number of times an arm from $A_s$ is sampled is $O(4^s n_s \log{(s/\delta)})$ for all $s$. \\\\
% \begin{proof}
% See Lemma 3.5. in \cite{karnin2013almost}.
% \end{proof}
Lemma 3.3. implies that the best arm is never eliminated. Furthermore, if multiple arms remain till the last round, the algorithm outputs the arm $i_r$ returned by \texttt{MedianElimination($S_r, \epsilon_r / 2, \delta_r$)}, with $r = s_\text{max}$. Since \texttt{MedianElimination($S, \epsilon, \delta$)} returns an $\epsilon$-optimal arm with probability $1 - \delta$ \cite{even2006action}, we get
\begin{align*} \mu_{i_r} &\geq \mu_* - \epsilon_r / 2 \\
&\geq \mu_* - \Delta/8 \,,
\end{align*}
which means that $i_r$ is the best arm.

We next calculate how many plays happen in a run of the algorithm. We only count the plays issued by the algorithm in line 5, since the same plays can be reused within invocations of \texttt{MedianElimination} in line 6. Let $T_r$ denote the number of plays in round $r$ of the algorithm. From Eq.~\ref{tsum}, we have:
\[ T_r = t_r + \sum_{i \in S_r} X_{r, i} \]
where $X_{r, i}$, the total reward for arm $a_i$ in the $r$th round, is distributed binomially with mean $\mu_i$ and number of trials $t_r$. Decomposing the sum over $S_r$ into sums over $S_{r,s}$, we get:
\[ T_r = t_r + \sum_{s=0}^{{s_{\text{max}}}} \sum_{i \in S_{r,s}} X_{r,i}\,. \]
Therefore, the total number of plays is given by:
\begin{align*}
    T &= \sum_{r=1}^{s_{\text{max}}} T_r \\
    &= \sum_{r=1}^{s_{\text{max}}} t_r + \sum_{r=1}^{s_{\text{max}}} \sum_{s=0}^{{s_{\text{max}}}} \sum_{i \in S_{r,s}} X_{r,i}\,,
\end{align*}
where in the first step, we have used the fact that the number of rounds is upper bounded by $s_{\text{max}}$. Now, $T$ is the sum of a constant term and $N = \sum_{r=1}^{s_\text{max}}\sum_{s=0}^{{s_{\text{max}}}} \sum_{i \in S_{r,s}} \lvert X_{r,i}\rvert$ independent Bernoulli terms.
Therefore, using Hoeffding's inequality, we get:
\begin{align}
\label{hoeff}
\mathbb{P} \left\{ T - \E{[T]} > \epsilon \right\} \leq
\exp{\left( - \frac{2 \epsilon^2}{N} \right)} \,.
\end{align}
We'll find an upper bound on
\[ \E{[T]} =\sum_{r=1}^{s_\text{max}} t_r + \sum_{r=1}^{s_\text{max}}\sum_{s=0}^{s_\text{max}} \sum_{i \in S_{r,s}} \mu_i t_r \,.\]
The first sum is given by
\begin{align*} \sum_{r=1}^{s_\text{max}} t_r &= 32 \sum_{r=1}^{s_\text{max}} \ln{\frac{100 r^3}{\delta}} = O\left(4^{s_\text{max}} \log{\frac{s_\text{max}}{\delta}}\right) \\
&= O\left( \frac{1}{\Delta^2} \log{\left(\frac{1}{\delta} \log{\frac{1}{\Delta}}\right)} \right)
\,.
\end{align*}
For the second sum, for $i \in S_{r,s}$, using $\mu_i = \mu_* - \Delta_i$ and $\Delta_i \geq 2^{-s}$, we get $\mu_i \leq \mu_* - 2^{-s}$. Therefore
\begin{align*}
\sum_{r=1}^{s_\text{max}}\sum_{s=0}^{s_\text{max}} \sum_{i \in S_{r,s}} \mu_i t_r  &\leq \sum_{r=1}^{s_\text{max}}\sum_{s=0}^{s_\text{max}} n_{r,s}(\mu_* - 2^{-s}) t_r\\
 &=  \sum_{s=0}^{s_\text{max}} (\mu_* - 2^{-s}) \left(\sum_{r=1}^{s_\text{max}} t_r n_{r,s} \right)\\
 &\stackrel{\text{(a)}}{\leq} \sum_{s=0}^{s_{\text{max}}}  (\mu_* - 2^{-s}) O\left( 4^s n_s \log{\frac{s}{\delta}} \right) \\
 &\stackrel{\text{(b)}}{=} O\left( \sum_{i \neq i^*} \frac{\mu_i}{\Delta_i^2} \log{\left(\frac{1}{\delta} \log{\frac{1}{\Delta_i}}\right)} \right) \,. \\
 \end{align*}
 where (a) follows from Lemma 3.5., and (b) from $\mu_* - 2^{-s} < 2 \mu_i$ and $\Delta_i < 2^{-s+1}$. Combining the two sums, we get
 \begin{align*}
 \E{[T]} &= O\left( \frac{1}{\Delta^2} \log{\left(\frac{1}{\delta} \log{\frac{1}{\Delta}}\right)} \right.\\
 &\quad \left. +  \sum_{i \neq i^*} \frac{\mu_i}{\Delta_i^2} \log{\left(\frac{1}{\delta} \log{\frac{1}{\Delta_i}}\right)}  \right) \,. \\
\end{align*}
Just like $\E{[T]}$, we can upper bound $N$ as
\begin{align*}
N &= \sum_{r=1}^{s_\text{max}} \left(t_r \sum_{s=0}^{s_{\text{max}}} n_{r,s} \right)\\
 &= O \left( \sum_{i \neq i^*} \frac{1}{\Delta_i^2} \log{\left(\frac{1}{\delta} \log{\frac{1}{\Delta_i}}\right)} \right) \,. \\
 \end{align*}
Setting $\epsilon = \sqrt{\frac{N}{2} \log{\frac{1}{\delta}}}$ in Eq.~\ref{hoeff} and using the upper bounds, we get the desired result.
\end{proof}

\subsection{Comparison with standard bandits}
Ignoring $1/\Delta_i$ terms (since they are small in comparison to $1\Delta_i^2$) and logarithmic factors, exponential gap elimination for linked bandits has a play complexity of $O\left(\Delta^{-2} + \sum_{i \neq i^*} \mu_i \Delta_i^{-2}\right)$. Compared to the sample complexity of standard bandits $O\left(\sum_{i \neq i^*} \Delta_i^{-2}\right)$, there is a $\mu_i$ dependent improvement, illustrating that our algorithm is taking advantage of the multiplicity of samples in linked bandits. Moreover, unlike \texttt{UniformSampling} which sampled all arms including those with large gaps equally, \texttt{LinkedEGE} samples each arm $O(1/\Delta_i^2)$ times. As we show next, the play complexity of our algorithm cannot be improved upon.

\section{Lower Bound for Linked Bandits}
% Change this to mannor's lower bound\\
In this section, we derive a problem-dependent lower bound on the number of plays needed in the fixed confidence setting. Our analysis relies on the classical work of \cite{mannor2004sample}, that provides a lower bound on the expected number of samples needed for each arm. We restate an abridged version of the theorem here.
\begin{theorem}
\label{lower}
\cite{mannor2004sample}. Consider the multi armed bandit with $n$ arms, with mean rewards $\mu_1, \mu_2, ..., \mu_n$. Let $\E{[t_i]}$ denote the expected number of trials of the $i$th arm, and $\Delta_i$ the gap $\mu^* - \mu_i$. Any policy that identifies the best arm with probability at least $1 - \delta$ for all mean rewards, satisfies
\[ \E{[t_i]} = \Omega\left(\frac{1}{\Delta_i^2} \right) \log{\frac{1}{\delta}} \]
for all $i \neq i^*$, and
\[ \E{[t_{i^*}]} = \Omega\left(\frac{1}{\Delta^2} \right) \log{\frac{1}{\delta}} \,,\]
where $\Delta = \min_{i \neq i^*} \Delta_i$.
\end{theorem}
% Our analysis follows the same direction as \cite{jamieson2014lil}, by relying on Farrell's optimal test \cite{farrell1964}.

% \begin{corollary}
% Consider the best arm identification in linked bandits problem with $n$ arms, with mean rewards $\mu_1, \mu_2, ..., \mu_n$. Let $\E{(t_i)}$ denote the expected number of trials of the $i$th arm, and $\Delta_i$ the gap $\mu^* - \mu_i$. Any procedure with $\sup_{\{\Delta_i \neq 0\}, i \neq i^*} \P{(\hat i \neq i^*)} \leq \delta, \delta \in (0, 1/2)$, necessarily has
% \[ \limsup_{\Delta_i \to 0} \frac{\E{(t_i)}}{\Delta_i^{-2} \log{\log{\Delta_i^{-2}}}} \geq 2 - 4 \delta \]
% for all $i \neq i^*$.
% \end{corollary}
% \begin{proof}
% See Corollary 1 in \cite{jamieson2014lil}.
% \end{proof}
The original theorem is stated for the problem of identifying an $\epsilon$-optimal arm. Setting $\epsilon = \Delta/2$ reduces it to best arm identification. As a corollary of this theorem, we obtain the lower bound for our problem.
\begin{corollary}
Consider the best arm identification in linked bandits problem with $n$ arms, with mean rewards $\mu_1, \mu_2, ..., \mu_n$. Let $\E{(T)}$ denote the expected number of plays. Any policy that identifies the best arm with probability at least $1 - \delta$ for all mean rewards, satisfies
\[ \E{[T]} = \Omega\left(\sum_{i \neq i^*}\frac{\mu_i}{\Delta_i^2} + \frac{\mu^* + p}{\Delta^2} \right) \log{\frac{1}{\delta}} \,, \]
where $p = \prod_{i=1}^n(1 - \mu_i)$.
\end{corollary}
\begin{proof}
Denoting the number of plays of the linked bandit by $T$, the number of trials of the $i$th arm with reward 1 by $u_i$ and the number of plays with no positive reward by $t_{n+1}$, we have
\[ T = \sum_{i=1}^n u_i + t_{n+1} \]
since each play ends with either exactly one positive reward, or with no positive reward.
Taking expectation on both sides, we get
\begin{align*}
    \E{(T)} &= \sum_{i = 1}^n \E{(u_i)} + \E{(t_{n+1})} \\
    % &\geq \sum_{i \neq i^*} \E{(u_i)} + \E{(t_{n+1})} \\
     &= \sum_{i = 1}^n \E_{t_i}{\E_{u_i | t_i}{(u_i)}} \\
     &\quad + \E{\left(\sum_{j=1}^T \P{(\text{no positive reward in }j\text{th play})}\right)}
     \end{align*}
 Using the fact that $u_i | t_i \sim \text{Binomial}(t_i, \mu_i)$, we get $\sum_{i = 1}^n \E_{t_i}{\E_{u_i | t_i}{(u_i)}} = \sum_{i = 1}^n \E_{t_i}{(\mu_i t_i)}$. Also, the probability of no positive reward has a minimum value $\prod_{i=1}^n (1 - \mu_i)$, achieved when all arms are selected in a play. Therefore, we get
     \begin{align*}
\E{[T]} &= \sum_{i = 1}^n \E_{t_i}{(\mu_i t_i)} + \E{(T)} \prod_{i= 1}^n (1 - \mu_i) \\
&= \sum_{i = 1}^n \mu_i \E{(t_i)} + \E{(T)} \prod_{i = 1}^n (1 - \mu_i) \\
 &\geq \sum_{i = 1}^n \mu_i \E{(t_i)} + \max_i \E{(t_i)} \prod_{i = 1}^n (1 - \mu_i) \,.
\end{align*} 
The last step follows from the fact that a single play can only sample a particular arm at most once, implying $T \geq t_i$ for all $i$. Combining this with the lower bounds from Theorem~\ref{lower}, we get the desired result.
\end{proof}
% Together with the corollary, this implies that no best arm procedure can have $\sup \P(\hat i \neq i^*) \leq \delta$ and use fewer than $(2 - 4 \delta) \left( \sum_{i \neq i^*} \mu_i \Delta_i^{-2} \log\log{\Delta_i^{-2}} + \mathcal{M} \Delta^{-2} \log\log{ \Delta^{-2}} \right)$ plays in expectation for all $\Delta_i$.
Other than the factor $\mu^* + p$ in the last term (and other small terms and factors), this lower bound matches the upper bound derived in Section~\ref{ege} for exponential gap elimination. Moreover, the looseness in one term can be ignored if the number of terms (arms) is large. The $\log{\log{1/\Delta}}$ factor was shown to be necessary by \cite{jamieson2014lil}, by relying on Farrell's optimal test \cite{farrell1964}. \\
% It is worth noting that despite the apparent looseness due to the factor $\mathcal{M}$, the overall upper and lower bounds are fairly tight for most values of $\mu_i$. This is because if the upper bound of $\sum_{i \neq i^*} \mu_i \Delta_i^{-2} \log\log{\Delta_i^{-2}} + \Delta^{-2} \log\log{ \Delta^{-2}}$ is dominated by the second term, the $\mu_i$'s have to be much smaller than $1$. This in turn implies that $\mathcal{M} \approx 1$, making the lower bound similar to the upper bound.

\subsection{Unordered linked bandits}
The lower bound that we just derived also holds for the more general problem, where the agent is allowed to re-order the selected arms in a play, i.e., the agent first selects a subset of arms, and then applies an arbitrary permutation. We call this the \emph{permuted linked bandit} setting. Moreover, we already have an optimal algorithm in the ordinary linked bandit setting that achieves this lower bound, and is also a valid algorithm for the permuted setting.
Since the sample complexity for permuted linked bandits is upper bounded by that of ordinary linked bandits, we have discovered a surprising result that ordered and unordered linked bandits have the same complexity.

% Check whether a tighter lower bound can be obtained by taking the last term containing E(T) and M on the left and dividing.
% For simulations, naive sampling strategies could be whole sequence allocation, uniform allocation, single arm selection, and deterministic shifting. For deterministic shifting, a discriminating example could be one where the best arm is preceded by many second-best arms, which prevent the best arm from getting sampled enough times with deterministic shifting. In contrast, our uniform sampling strategy will still sample all arms equally.
% Have a section explaining the significance of the dependence on mu_i's of sample complexity, and how it is different from (better than) the analysis of Kveton et al, who derive regret bounds which don't depend on the mu_i's of the optimal arms.

\section{Experiments}
\begin{figure}
\vskip 0.2in
\begin{center}
\centerline{\includegraphics[trim={0 0 17cm 0cm},clip,width=1.0\columnwidth]{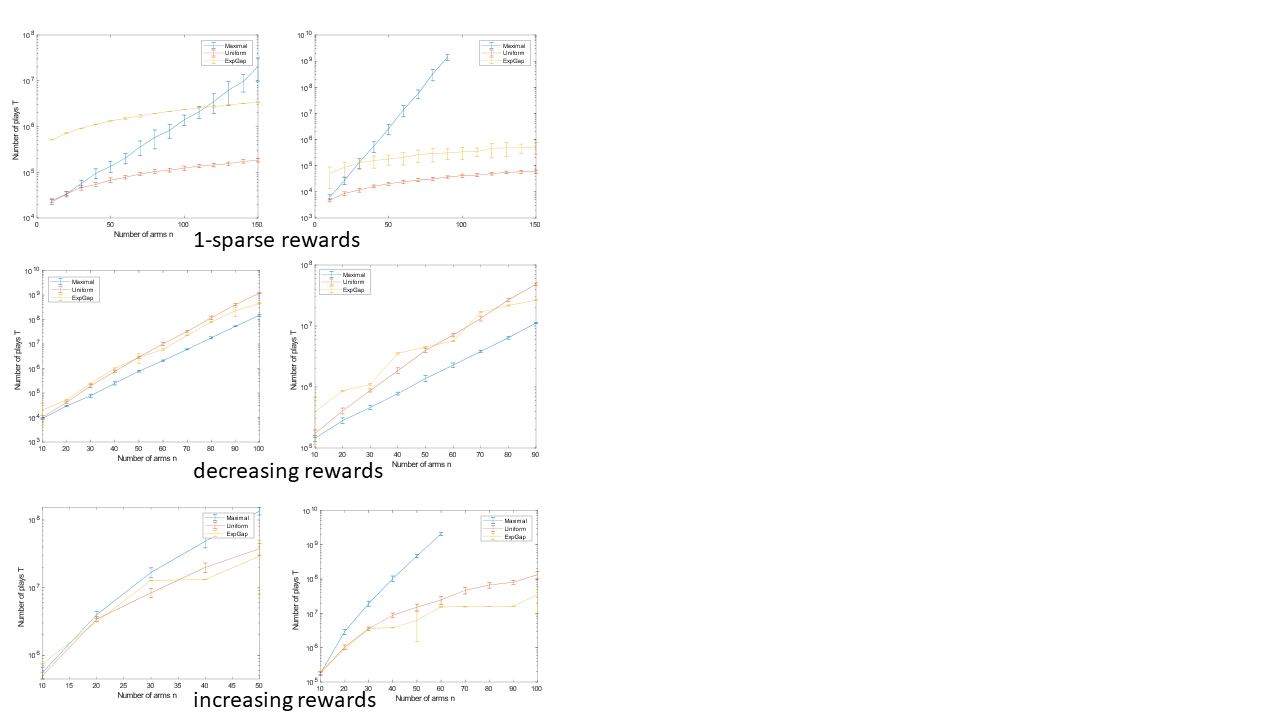}}
\caption{Sample complexities of proposed algorithms across various problem scenarios.}
\label{icml-historical}
\end{center}
\vskip -0.2in
\end{figure}
% We conduct an experiment to compare the sample compe the naive algorithms to the optimal one, and shows that across a wide range of scenarios, the optimal \texttt{LinkedEGE} algorithm indeed performs better than the others in terms of play complexity. Second, we demonstrate the relation between the play complexity of linked bandits and the sample complexity of standard bandits. Our experiment shows that the empirical play complexity shows the same dependence on the problem parameters as the predicted upper and lower bounds, and improves upon the sample complexity of standard bandits by the predicted amount.

% \subsection{Comparison between different sampling strategies}
We conduct an experiment to compare the play complexities of the 3 sampling strategies described in this paper: \texttt{MaximalSampling}, \texttt{UniformSampling} and \texttt{LinkedEGE}. We consider three problem scenarios over increasing number of arms. The ``1-sparse'' scenario sets $\mu^* = 0.1$ and $\mu_i = 0.05$ for $i \neq i^*$. The ``decreasing'' scenario sets $\mu_1 = 0.05$ and $\mu_i = 0.05 - 0.005 * 0.95^{(n-i)/2}$ for $i \neq 1$. The ``increasing'' scenario sets $\mu_i = i / n$ for $1 \leq i \leq n$. Figure ~\ref{icml-historical} shows the comparison. In the 1-sparse scenario (first row), uniform sampling and linked EGE have similar performance, with an optimal dependence on $n$. This is expected since both of them sample arms uniformly in this scenario, which is the optimal strategy here. The decreasing scenario (second row) is constructed such that maximal sampling algorithm will sample arms at a rate proportional to $\Delta_i^{-2}$, making it the optimal strategy. Clearly, maximal sampling dominates the other two in this scenario. While linked EGE seems to be sub-optimal in this case, it is only due to the large constants in its sample complexity. As $n$ increases, it approaches maximal sampling in its performance. In the increasing scenario (third row), linked EGE has the best empirical performance, owing to its gap-dependent sampling strategy. For maximal sampling and uniform sampling, we used the LIL stopping criterion described by \cite{jamieson2014lil}.

% \subsection{Comparison between play complexity and sample complexity}

% % \subsection{UCB-type algorithm}
% We conduct an experiment to illustrate the difficulty of designing a UCB-type algorithm for linked bandits, that achieves optimal sample complexity. We consider an algorithm that selects the suffix of arms starting at the highest UCB in each play. This algorithm has a strictly better sample complexity than the standard MAB, since the set of arms that it samples in each play is a superset of the set in the case of standard MAB. However, it doesn't allocate the samples from each play to arms in proportion to their need, instead favoring the arms following the highest UCB arm. Figure shows the number of plays used by the algorithm to identify the best arm across a number of adversarial problems. The algorithm has a consistently higher play complexity than exponential gap elimination, often by a large factor. Since our algorithm doesn't eliminate sub-optimal arms at any time, the extra samples from each play are unnecessarily used up by arms that have a small UCB.
% We also show two attempted UCB-type algorithms that always terminate, but incur a much higher complexity than optimal in the worst case. Their sample complexity is only slightly lower than the standard version, since they don't reap the benefit of multiple samples per play. They never eliminate arms, which take up more samples than needed and prevent the arms which actually need more samples from being sampled. \\

\section{Related work}
The standard version of the best arm problem has been extensively studied since the '50s, especially in the last decade. Two settings that are commonly studied: \textit{fixed confidence} and \textit{fixed budget}, in which either the error probability or the number of arm pulls are fixed and the goal is to minimize the other (see \cite{gabillon:hal-00747005} for similarities between the two settings). In the fixed confidence setting, the uniform sampling strategy has been studied as a precursor to optimal strategies, and shown by \cite{even2006action,BUBECK20111832} to achieve a sample complexity of order $n \Delta^{-2} \log{(n/\delta)}$. Based on the classification by \cite{jamieson2014best}, sampling strategies that are optimal or nearly optimal can be broadly divided into two classes: action elimination (AE) and upper confidence bound (UCB) based algorithms. Among AE algorithms, the exponential-gap elimination procedure of \cite{karnin2013almost} guarantees best arm identification with high probability using order $\sum_i \Delta_i^{-2} \log{(1/\delta \log{\Delta_i^{-2}})}$ samples, which is optimal. Among UCB algorithms, lil' UCB \cite{jamieson2014lil} also achieves optimal sample complexity. While AE algorithms proceed in epochs performing uniform sampling of the remaining arms and elimination of sub-optimal arms in each epoch, UCB algorithms use an upper confidence bound to define an ordering for all arms, and choose the single highest arm to sample at each time. Our work shows that AE algorithms can be easily adapted to linked bandits by using an efficient subroutine to perform uniform sampling. In some sense, AE algorithms are agnostic to the order of sampling, which can be chosen to match the fixed arm order of linked bandits. In contrast, the dynamically changing arm order in UCB algorithms makes them less suited for adaption to linked bandits.

% algorithms do not eliminate sub-optimal arms, they are less suitable for adaption to linked bandits \\
Our problem is closely related to the one studied by \cite{kveton2015cascading}, cascading bandits, to model the problem of learning to rank. Both cascading bandits and linked bandits share a similar feedback mechanism, i.e., arms are sampled upto the first non-zero reward. However, there are some key differences. The goal in cascading bandits is to select an ordering of arms that maximizes cumulative reward, or equivalently, minimizes cumulative regret. In linked bandits, the goal is to identify the maximal reward arm with high probability. Therefore, obtaining feedback that discriminates between arms is more important than the reward.

% We extend the idea to the best arm identification setting, which we show is naturally conducive to linked bandits.
Many other variants of the bandit problem have also been proposed. Of particular relevance are combinatorial bandits (\cite{chen2013combinatorial}), which deal with arms that form certain combinatorial structures, and in each round, a set of arms (\textit{super arm}) are played together. The semi-bandit feedback setting of  (\cite{neu2013efficient}) is similar to linked bandits in that it provides more feedback in a single interaction than standard bandits. However, the key difference is that the feedback here is provided for a stochastic set of arms that can't be chosen by the agent in a deterministic manner.
% All these results can be considered as steps toward the generalized combinatorial bandit problem where the action, the reward and the feedback, all come from certain combinations of the set of arms. 

\section{Discussion and Future work}
We looked at 2 algorithms for best arm identification in the standard MAB problem, and adapted them to the linked MAB problem with good results. The tool which allowed us to do this was a method to perform uniform sampling efficiently.
% Indeed it was this uniform sampling strategy that helped avoid the wastage associated with naive sampling, and reap the full benefit of linked bandits as compared to standard bandits.
Since the uniform sampling subroutine is ``optimal'' in terms of sample complexity, combining it with an optimal standard MAB algorithm, namely exponential gap elimination, gives a close to optimal sample complexity for the best arm identification problem in linked bandits.
We saw that our results for linked bandits carry over to the setting where the bandits can be re-ordered arbitrarily. This can be useful in applications like online ad placement which have a linked feedback mechanism.

Finally, the discussion of linked bandits can be extended to consider the problem of identifying the $k$-best arms with linked bandits. This will be especially useful for the search ranking problem, where we are interested in identifying the  top-$k$ search results for a query.
% \section{Future work}
% In this paper, we looked at two allocation strategies for the linked bandits problem. We pWe provided some justification for why exponential gap elimination with uniform sampling is optimal in terms of sample complexity. A formal lower bound is still missing, and is desirable for completeness. \\
% The formulation of linked bandits that we considered is still restricted in the sense that the player selects a subsequence of arms rather than an arbitrarily ordered subset. Many real life problems don't have this restriction, particularly search result ranking and ad placement. \\

% \bibliographystyle{unsrtnat}
\bibliography{example_paper}

% \subsection{Figures}

% You may want to include figures in the paper to illustrate
% your approach and results. Such artwork should be centered,
% legible, and separated from the text. Lines should be dark and at
% least 0.5~points thick for purposes of reproduction, and text should
% not appear on a gray background.

% Label all distinct components of each figure. If the figure takes the
% form of a graph, then give a name for each axis and include a legend
% that briefly describes each curve. Do not include a title inside the
% figure; instead, the caption should serve this function.

% Number figures sequentially, placing the figure number and caption
% \emph{after} the graphics, with at least 0.1~inches of space before
% the caption and 0.1~inches after it, as in
% Figure~\ref{icml-historical}. The figure caption should be set in
% 9~point type and centered unless it runs two or more lines, in which
% case it should be flush left. You may float figures to the top or
% bottom of a column, and you may set wide figures across both columns
% (use the environment \texttt{figure*} in \LaTeX). Always place
% two-column figures at the top or bottom of the page.

\nocite{combes2015extension}
\nocite{bickel2009springer}
\bibliographystyle{icml2019}
\clearpage
\onecolumn
\renewcommand{\figurename}{Supplementary Figure}
\renewcommand{\thesection}{A.}
\renewcommand{\theequation}{\arabic{equation}}
\setcounter{figure}{0}
\setcounter{section}{0}
\setcounter{equation}{0}
\setcounter{page}{1}
\newtheorem{lemma}[theorem]{Lemma}

% \begin{center}
% %\begin{tabular}{c}
% \huge Supplementary Document for\\[0.2cm]
% \huge ``Best Arm Identification in Linked Bandits'' \\[1.1cm]
% %\normalsize Atul Ingle, Andreas Velten, Mohit Gupta.\\[0.5cm]
% %Correspondence to: ingle@uwalumni.com
% \normalsize Anonymous ICML Submission\\
% %\end{tabular}
% \end{center}

\section{Proof of Theorem 3.1.}
In this section, we'll derive an upper bound on the sample complexity of the \texttt{MaximalSampling} strategy for linked bandits. \\
{\bf Notation: }Let $T$ denote the number of plays. Let $X_j \in \{0,1\}^n$, $ 1 \leq j \leq T$ denote 1 hot reward vectors for the whole sequence of arms selected in the $j$th play, extended with $0$s for the arms that did not get sampled. Let $\mu_i, 1 \leq i \leq n$ denote the true means.\\
We define the mean estimates as
\[ \hat \mu_{i,T} = \frac{\sum_{j=1}^T X_j[i]}{T - \sum_{k=1}^{i-1}\sum_{j=1}^T X_j[k]} \,.\]
Note that this is the same definition as earlier, except that the form is more amenable to application of McDiarmid's inequality. The algorithm outputs the arm $\hat i = \argmax_i{\{\hat \mu_{i,T}\}}$. Therefore, we want to bound the probability of error:
\[ \mathbb{P}\left\{\max_{i:\Delta_i > 0} \hat \mu_{i,T} \geq \hat \mu_{i^*, T}\right\} \]
Define $f(X_1, X_2, ..., X_n) = \max_{i:\Delta_i > 0} \hat \mu_{i,T} - \hat \mu_{i^*, T} $. $f$ is a function of independent random variables, and it is easy to see that $f$ has bounded differences:
\[ |f(x_1, ... , x_n) - f(x_1, ..., x_{i-1}, x_i^{\prime}, x_{i+1}, ..., x_n) | \leq \frac{1}{b} \]
where $b = T - \sum_{k=1}^{n-1}\sum_{j=1}^T X_j[i]$, the smallest denominator amongst all the quotients. Therefore, $f$ can be shown to satisfy a concentration inequality. 
% Therefore, $f$ concentrates around its mean:
\begin{lemma}
\label{lemma1}
For every $\epsilon > 0$ and $c > \frac{1}{Tp}$, let $q =\exp{\left(- \frac{2( p T c - 1)^2}{T c^2}\right)}$. Then
\[ \mathbb{P} \{ f(X) - \mathbb{E}\{f(X | \Theta\} \geq \epsilon \} \leq q + \exp{\left(- \frac{2(\epsilon - q T c)^2}{T c^2}\right)} \]
where $\Theta$ represents the event $1/b < c$, provided $\epsilon > q T c $.
\end{lemma}
\begin{proof}
Now, b is binomially distributed as $b = B\left(T, p\right)$ with $p = \prod_{i=1}^{n-1} (1 - \mu_i)$. Therefore, we can bound $b$ w.h.p around its mean using Hoeffding's inequality:
\[ \mathbb{P}\{b < Tp(1 - \delta)\} < e^{-2 T p^2 \delta^2} \]
where $0 < \delta < 1$.
\[ \implies \mathbb{P}\left\{\frac{1}{b} > c \right\} < q \]
where $c = \frac{1}{(1-k)Tp}$.
Therefore, we can use McDiarmid's inequality for differences bounded w.h.p. due to \cite{combes2015extension} to get the desired result.
\end{proof}
% \begin{equation} \label{apeq2} = \exp{\left(-2 T p^2 k^2 \right)} + \exp{\left(- 2T p^2 (1-k)^2 (\epsilon - q T c)^2 \right) } \end{equation}
%Setting $k = \epsilon$, we get:
%\begin{equation} \label{apeq2}
% \mathbb{P} \{ f(X) - m \geq \epsilon \} \leq O(\exp{\left(-2 T p^2 \epsilon^2 \right)})
% \end{equation}
%provided $\epsilon (1 - \epsilon) p >> e^{- 2 T p^2 \epsilon^2}$ (or equivalently, $T >> $). \\
We next find an upper bound on $\mathbb{E}\{f(X) | \Theta\}$.
\begin{lemma}
\label{lemma2}
\[ \mathbb{E}\{f(X) | \Theta\} < \sqrt{{c\ln n}} - \Delta \]
\end{lemma}
\begin{proof}
Define $Z_i = \hat \mu_i - \hat \mu_{i^*} + \Delta_i =(\mu_{i^*} - \hat \mu_{i^*}) + (\hat \mu_i - \mu_i)$. Note that we have dropped the $T$ subscript from $\hat \mu$ for clarity. With this definition, we have
\begin{align}
\label{first}
\mathbb{E}\{f(X) | \Theta\} &= \max_{i: \Delta_i > 0} \left( - \Delta_i + \mathbb{E} \{ Z_i | \Theta \} \right) \nonumber \\
    &\leq - \min_{i: \Delta_i > 0} \Delta_i + \mathbb{E} \{ \max_{i: \Delta_i > 0} Z_i | \Theta \}
\end{align}
Assuming $i > i^*$ (the other case can be analyzed similarly), we get for all $s > 0$
\begin{align} \label{apeq1}
\mathbb{E} \{ e^{s Z_i} | \Theta \} &= \mathbb{E} \left[ e^{s ( \mu_{i^*} - \hat \mu_{i^*})} e^{s (\hat \mu_{i} - \mu_{i})} | \Theta \right] \nonumber \\
&= \mathbb{E} \left[ e^{s ( \mu_{i^*} - \hat \mu_{i^*})} \mathbb{E} \{ e^{s (\hat \mu_{i} - \mu_{i})} | \hat \mu_{1 ... i^*}, \Theta \} | \Theta \right]
\end{align}
using the law of total expectation, where $\hat \mu_{1 ... j}$ denotes $\hat \mu_1, \hat \mu_2, ... , \hat \mu_j$.
Let's look at the second term inside the outer expectation. Again using the law of total expectation, we get
\begin{align}
\label{total}
\mathbb{E} \{e^{s (\hat \mu_{i} - \mu_{i})} | \hat \mu_{1 ... i^*}, \Theta\} &= \mathbb{E}_{\hat \mu_{1 ... {i-1}} | \hat \mu_{1 ... i^*}, \Theta} \left[\mathbb{E}_{\hat \mu_i | \hat \mu_{1 ... {i-1}}, \Theta} \{e^{s (\hat \mu_{i} - \mu_{i})} \}\right] \nonumber \\
 &< \mathbb{E}_{\hat \mu_{1 ... {i-1}} | \hat \mu_{1 ... i^*}, \Theta} \left[\mathbb{E}_{\hat \mu_i | \hat \mu_{1 ... {i-1}}} \{e^{s (\hat \mu_{i} - \mu_{i})} \}\right] 
\end{align}
where the last step uses the fact that conditioning on $\Theta$ reduces the probability of large $\hat \mu_i$, since $\hat \mu_i$ is negatively correlated with $b$.\\
Now, $\hat \mu_i | \hat \mu_{1 ... {i-1}} \sim \text{Bin}(b_{i-1}, \mu_i)$, where $b_{i-1} = T \prod_{k=1}^{i-1} (1 - \hat \mu_i)$. Using the expression for the moment generation function of the binomial distribution, we get
\begin{align*} \mathbb{E}_{\hat \mu_i | \hat \mu_{1 ... {i-1}}} \{e^{s (\hat \mu_{i} - \mu_{i})} \} &= \mathbb{E}_{ \text{Bin}(b_{i-1}, \mu_i)} \{e^{s (\hat \mu_{i} - \mu_{i})} \}\\
&= \exp{\left\{\left((e^{s/b_{i-1}} - 1)b_{i-1} - s\right)p_i\right\}} \\
&< \exp{\left\{s^2 p_i / b_{i-1}\right\}}
\end{align*}
for $s \leq b_{i-1}$. Combining with Eq.~\ref{total}, we get
\begin{align*}
\label{mui}
 \mathbb{E} \{e^{s (\hat \mu_{i} - \mu_{i})} | \hat \mu_{1 ... i^*}, \Theta\} &< \mathbb{E}_{\hat \mu_{1 ... {i-1}} | \hat \mu_{1 ... i^*}, \Theta} \left[ e^{s^2 p_i / b} \right] < e^{s^2 {p_i c}}  \leq e^{{s^2 c}} \,,
 \end{align*}
since $1/b_{i-1} \leq 1/b < c$ conditioned on $\Theta$.
By similar reasoning, we get
\[ \mathbb{E} \left[ e^{s ( \mu_{i^*} - \hat \mu_{i^*})} | \Theta \right] < e^{{s^2 c}} \]
Combining the above with Eq.~\ref{apeq1}, we get:
\[ \mathbb{E} \{ e^{s Z_i} | \Theta \} < e^{{2 s^2 c} } \]
for $s \leq 1/c$.
Thus, $Z_i$ are sub-gaussian. Using the inequality for max of sub-gaussian random variables \cite[Chapter~2]{bickel2009springer}, we get:
\[ \mathbb{E} \left[ \max_{i=1,...,n} Z_i | \Theta \right] < 2\sqrt{{2c \ln{n}}} \]
provided $c \ln{n} \leq 2$.
Finally, combining this with Eq.~\ref{first} completes the proof.
\end{proof}
Setting
\[ \epsilon =  \Delta - 2\sqrt{2 c \ln{n}} \]
in Lemma \ref{lemma1} and combining with Lemma \ref{lemma2}, we get:
\[ \mathbb{P} \left\{ \max_{i: \Delta_i > 0} \hat \mu_i \geq \hat \mu_{i^*} \right\} \leq q + \exp{\left(- \frac{2(\Delta - 2\sqrt{2 c \ln{n}} - q T c)^2}{T c^2}\right)}   \]
Setting $c = \frac{1 + \Delta}{T p}$ and $T = \frac{1}{\Delta^2 p^2} \log{\frac{n}{\Delta p \delta}}$, we get
\[ \mathbb{P} \left\{ \max_{i: \Delta_i > 0} \hat \mu_i \geq \hat \mu_{i^*} \right\} \leq \delta \,.\]
% \[ = \exp{\left(-2 T p^2 k^2 \right)} + \exp{\left(- 2T p^2 (1-k)^2 \left(\Delta - 2\sqrt{\frac{2\ln{n}}{Tp(1 - k)}} - \frac{e^{-2 T p^2 k^2}} {p(1 - k)}\right)^2 \right) } \]
% where $\Delta = \min_{i: \Delta_i > 0} \Delta_i$. For appropriate value of $k$, and under some mild conditions on $T$, we can show the sample complexity result provided in the text.

\end{document}